\documentclass{article}

\usepackage[accepted]{icml2026}

\usepackage[utf8]{inputenc} 
\usepackage[T1]{fontenc}    
\usepackage{hyperref}       
\usepackage{url}            
\usepackage{booktabs}       
\usepackage{amsfonts}       
\usepackage{nicefrac}       
\usepackage{microtype}      
\usepackage{lipsum}
\usepackage{fancyhdr}       
\usepackage{graphicx} 
\usepackage{amsmath,amssymb,amsthm}
\usepackage{xcolor}
\usepackage{cite}
\usepackage{algpseudocode}
\usepackage{mathtools}
\usepackage{nccmath}
\graphicspath{{media/}}     

\newcommand{\E}{\mathbb{E}}

\newcommand{\Gtm}{g_{t,m}}

\newcommand{\Gim}{g_{i,m}}
\newcommand{\Gil}{g_{i,l}}
\newcommand{\vGtm}{\bar{g}_{t,m}}

\newcommand{\algoOFL}{\text{FedSEA}}
\newtheorem{theorem}{Theorem}
\newtheorem{lemma}{Lemma}

\newtheorem{assumption}{Assumption}
\newtheorem{remark}{Remark}
\title{FedSEA: Achieving Benefit of Parallelization in Federated Online Learning
}

\author{
  Harekrushna Sahu \quad Pratik Jawanpuria \quad Pranay Sharma \\
  CMInDS, IIT Bombay, India \\
  \texttt{\{sahu.hk, pratik.jawanpuria, pranaysh\}@iitb.ac.in} \\
}

\begin{document}
\maketitle
\thispagestyle{empty}


\begin{abstract}
Online federated learning (OFL) has emerged as a popular framework for decentralized decision-making over continuous data streams without compromising client privacy. However, the adversary model assumed in standard OFL typically precludes any potential benefits of parallelization. Further, it fails to adequately capture the different sources of statistical variation in OFL problems. In this paper, we extend the OFL paradigm by integrating a stochastically extended adversary (SEA). Under this framework, the loss function remains fixed across clients over time. However, the adversary dynamically and independently selects the data distribution for each client at each time. We propose the \algoOFL{} algorithm to solve this problem, which utilizes online stochastic gradient descent at the clients, along with periodic global aggregation via the server. We establish bounds on the global network regret over a time horizon \(T\) for two classes of functions: (1) for smooth and convex losses, we prove an \(\mathcal{O}(\sqrt{T})\) bound, and (2) for smooth and strongly convex losses, we prove an \(\mathcal{O}(\log T)\) bound. Through careful analysis, we quantify the individual impact of both spatial (across clients) and temporal (over time) data heterogeneity on the regret bounds. Consequently, we identify a regime of mild temporal variation (relative to stochastic gradient variance), where the network regret improves with parallelization. Hence, in the SEA setting, our results improve the existing pessimistic worst-case results in online federated learning.
\end{abstract}

\section{Introduction}
\paragraph*{Online Learning.}
With the rapid growth of big data, traditional batch learning paradigms face significant limitations when data arrives and evolves continuously. In batch learning, the primary aim is to learn a model using a static dataset that generalizes well to unseen data. However, real-world applications such as financial market forecasting, autonomous vehicles, and recommendation systems \citep{hazan2023introductiononlineconvexoptimization} often get continuous streams of data, and the underlying distribution shifts over time \citep{conceptdrift}. The online learning (OL) paradigm addresses this problem of learning in dynamic environments. Instead of retraining the model from scratch, an online learner updates its model incrementally as new data arrives \citep{hazan2023introductiononlineconvexoptimization, orabona2025modernintroductiononlinelearning}. Unlike batch learning, the goal in OL is to produce a sequence of models that minimizes the cumulative regret of the learner, defined as the difference between the cumulative loss of the learner and that of an idealized comparator with prior knowledge of the system dynamics \citep{OL_survey}.

\paragraph*{The Shift to Federated Learning.} The standard OL frameworks assume that the incoming data streams can be collected at a centralized server. As data generation rapidly shifts to decentralized edge devices \citep{Edgecomputingsurvey}, centralization introduces significant challenges such as increased communication overhead and severe privacy concerns. 
Federated learning (FL) mitigates these issues by ensuring that the raw data never leaves the client devices \citep{FL_McMahana}. The clients use their local data to train models and communicate with the server intermittently, which in turn periodically aggregates these local updates into a global model \citep{FL_McMahana}. By distributing the training process, FL ensures data privacy and reduces communication overhead by leveraging local computational resources \citep{FL_monograpg_kairouz}. FL has been successfully used in various sectors like healthcare, Internet of Things (IoT), and decentralized wireless communication networks \citep{FL_overview_strategies}. However, the existing work in FL primarily focuses on batch learning problems \citep{FL_McMahana, SCAFFOLD, fedprox, khaled_tighter_theory_of_LSGD, stich2019localsgdconvergesfast}. 

\paragraph*{Online Federated Learning (OFL).}
Some recent works explore FL for continuous data streams \citep{OFL_mitra, patel_OFL_with_bandit_feedback}. In standard OFL, clients do not hold a static dataset. Instead, an adversary sequentially dictates the loss functions for each client in each round \citep{OFL_mitra, ganguly_OFL_nonstationary}. Although there are two sources of variation in the data distribution: spatial (across clients) and temporal (over time), existing regret bounds often fail to explicitly capture their individual impact. This is due to the obscuring effect of the restrictive bounded gradient (Lipschitz) assumption \citep{OFL_mitra, patel_OFL_with_bandit_feedback}. While recent work by Kwon et al. \citep{kwon2023tighter} captures the spatial heterogeneity by relaxing the bounded gradient assumption, capturing both spatial and temporal heterogeneity simultaneously remains an open challenge.

Furthermore, a significant limitation in much of the existing OFL literature is the reliance on true (sub)gradient computation \citep{OFL_mitra, ganguly_OFL_nonstationary}. This is often impractical for applications that consistently generate massive volumes of data, such as real-time recommendation systems. On the other hand, using stochastic gradients is computationally cheap. In addition, averaging these stochastic gradient-based updates across clients results in reduced error due to variance reduction. Consequently, the speed of learning improves with an increasing number of participating clients \citep{patel_OFL_with_bandit_feedback}.

To overcome the limitations of full gradient feedback and capture the benefits of linear speedup resulting from parallelization, we need more realistic assumptions. The Stochastically Extended Adversary (SEA) framework, initially introduced by Sachs et al.~\citep{sachs_SEA} provides a robust alternative. In the SEA setting, the learner and the adversary make simultaneous moves at each time step: the learner plays a decision, and the adversary selects a data distribution. Both parties then observe a random sample drawn from the chosen distribution, and the learner incurs a loss based on this sample \citep{sachs_SEA, Chen_SEA_OMD}. 
The SEA framework gracefully bridges the gap between stochastic optimization (adversary is restricted to a single fixed distribution) and adversarial online optimization (adversary selects Dirac delta distributions). 

\begin{figure}[h!]
    \centering
    \includegraphics[width=0.8\linewidth]{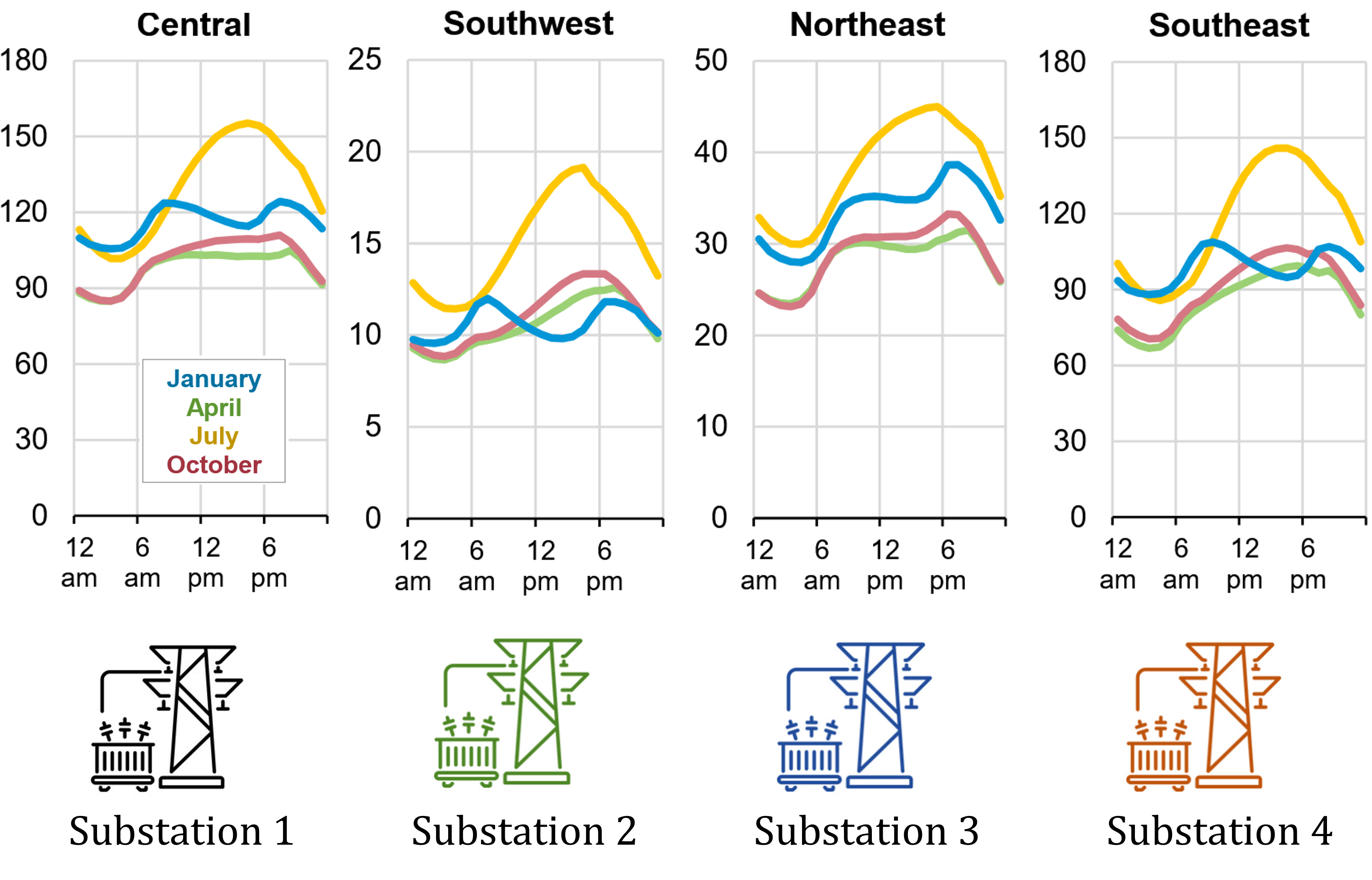}
    \caption{Illustration of online federated learning in load forecasting in smart grids. The substations can be modeled as different clients, each with its unique load distribution. The curves above show the average hourly electricity load (in million kWh) in 4 regions of the US during a typical day in a few selected months (Source: \href{https://www.eia.gov}{U.S. Energy Information Administration}). The \textit{spatial} variations in the load are evident from the y-axis scaling in the four regions. The \textit{temporal} variations depend on the time of day and month, with the highest load during summer afternoons.}
    \label{fig:smargrid}
\end{figure}
\paragraph*{Proposed framework - OFL with SEA.}
We propose integrating the OFL framework with SEA. Under the proposed setting, the clients receive continuously evolving data distributions, chosen by the SEA. By modeling clients as individually operating within SEA environments, we can accurately capture both the stochastic variance inherent in physical data observations and temporal distribution shifts. 
The proposed setting is useful for modeling complex real-world decentralized systems, such as recommendation engines or load forecasting in a smart grid, where a global model must aggregate knowledge from clients, facing both spatial heterogeneity and temporal evolution simultaneously. See Figure \ref{fig:smargrid} for an illustration of the spatial and temporal heterogeneity inherent in a smart grid application. Our main contributions are summarized as follows:
\begin{itemize}
    \item \textbf{Integration of SEA with OFL:} We formalize a novel distributed online learning framework incorporating a Stochastically Extended Adversary (SEA) and propose the \algoOFL{} algorithm (Algorithm \ref{O_FedAvg}). Our framework allows us to capture the individual effects of spatial (across clients) and temporal (across time) variations in data distributions.
    \item \textbf{Regret Bounds:} For smooth and convex loss functions, our algorithm achieves the order-optimal \(\mathcal{O}(\sqrt{T})\) regret bound (Theorem \ref{theorem 1}). For smooth and strongly convex loss functions, our algorithm achieves the order-optimal logarithmic \(\mathcal{O}(\log T)\) regret bound (Theorem \ref{theorem 2}).
    \item \textbf{Provable Benefit of Parallelization:} Our fine-grained analysis allows us to identify a regime of mild temporal variation where the network regret bound improves with parallelization. Therefore, in the presence of a weaker adversary, our results improve the pessimistic worst-case results in online federated learning \citep{patel_OFL_with_bandit_feedback}.
\end{itemize}

\section{Preliminaries and Problem Setup}
\paragraph*{Online Learning.}
In a standard online learning (OL) framework, at each time step \(t\in[T]\triangleq\{1,\ldots,T\}\), a learner plays the decision \(x_t\) from a convex feasible set \(\mathcal{X}\). Subsequently, the adversary reveals a loss function \(f_t\), and the learner incurs a loss $f_t(x_t)$. The objective is to learn a sequence of models that minimizes the cumulative regret \(R_T\) over the time horizon $T$, measured against a fixed best-in-hindsight decision $x^*=\min_{x\in\mathcal{X}} \sum_{t=1}^Tf_t(x)$ \citep{hazan2023introductiononlineconvexoptimization,orabona2025modernintroductiononlinelearning}: 
\begin{equation}
\label{eq:regret_defn}
    R_T=\sum_{t=1}^Tf_t(x_t) - \sum_{t=1}^Tf_t(x^*). 
\end{equation}
However, the aim of (batch) federated learning (FL) is to learn a fixed model that minimizes a \textit{static} global loss \citep{FL_monograpg_kairouz}.
\begin{equation}
\label{eq:FL_static_problem}
    g(x) \triangleq \frac{1}{M} \sum_{m=1}^{M} \Big\{ g_m(x) \coloneqq \mathbb{E}_{\xi_m \sim \mathcal{D}_m}\big[ g(x; \xi_m) \big] \Big\},
\end{equation}
where $[M]\triangleq\{1,\dots,M\}$ is the set of clients and  $\mathcal{D}_m$ represents the static data distribution at client $m$. 

Online Federated Learning (OFL) generalizes FL to dynamic, sequential settings \citep{OFL_mitra, patel_OFL_with_bandit_feedback}. At each time step \(t\), each client $m \in [M]$ plays a decision \(x_{t,m}\in\mathcal{X}\). Subsequently, client $m$ observes a loss function $f_{t,m}$ and incurs a loss \(f_{t,m}(x_{t,m})\). We note that unlike centralized OL where a single global loss $f_t$ is revealed, OFL introduces client-specific loss functions. The global loss at time $t$ is the average of these local losses, defined as \(f_t(x)=\frac1M\sum_{m=1}^Mf_{t,m}(x)\). 

\paragraph*{Federated Learning in SEA environment (Proposed setting).}
Given a central server and $M$ clients, we study the Stochastically Extended Adversary (SEA) model within an OFL framework.
At each time $t$, client $m \ (\in[M])$ plays a \textit{local} decision $x_{t,m}\in\mathcal{X}$. Simultaneously, an adversary independently chooses data distributions $\{ D_{t,m} \}_{m\in[M]}$  across the \(M\) clients. 
Next, both the adversary and client $m$ observe a sample $\xi_{t,m} \sim D_{t,m}$, and client \(m\) incurs a loss $f(x_{t,m},\xi_{t,m})$. In this setting, the functional form of the loss $f$ is fixed, but the underlying data distributions are dictated by the adversary.
\paragraph*{Performance Metric.}
Following existing OFL work \citep{OFL_mitra, patel_OFL_with_bandit_feedback}, the empirical cumulative regret $R_T$ of the network over the time horizon $T$ is measured by evaluating the performance of every client's local model against the data samples of the entire network. 
Since the data samples are drawn from distributions chosen by the SEA, the network objective is to minimize the expected cumulative regret. To formalize this, let the expected local loss for client $m$ at time $t$ be $f_{t,m}(x) = \E_{\xi \sim D_{t,m}}[f(x, \xi)]$, and the expected global loss across the network be $f_t(x) = \frac{1}{M} \sum_{m=1}^M f_{t,m}(x)$. Then the expected cumulative regret against the optimal fixed competitor $x^*$ in hindsight is generalized to:
\begin{equation}\label{regret equ}
    \E[R_T]=\frac1M\sum_{t=1}^T\sum_{m=1}^M\E[f_t(x_{t,m})]-\sum_{t=1}^Tf_t(x^*)
\end{equation}

\paragraph*{Proposed Algorithm.}
To minimize the expected regret  \eqref{regret equ}, we propose \algoOFL{} (Algorithm \ref{O_FedAvg}).
After taking action $x_{t,m}$ and incurring the loss $f(x_{t,m},\xi_{t,m})$ (steps \ref{algo:step_action}-\ref{algo:step_loss} in Algorithm \ref{O_FedAvg}), each client $m$ updates its local model via (projected) stochastic gradient descent using the stochastic gradient $\nabla f(x_{t,m},\xi_{t,m})$  with step size $\eta_t$. The clients communicate with the server intermittently to aggregate the global information.
Once every $\tau$ time steps, the clients send their current local models to the server, which computes the average and broadcasts the updated global model back to the clients (steps \ref{algo:step_avg1}-\ref{algo:step_avg3} in Algorithm \ref{O_FedAvg}).

\begin{algorithm}[H]
\caption{Federated Averaging (Online \algoOFL{})}
\label{O_FedAvg}
\small
\begin{algorithmic}[1]
\State \textbf{Input:} Clients $[M] \triangleq \{1,\dots,M\}$, horizon $T$, set $\mathcal{X}\subseteq \mathbb{R}^d$, a sequence of step sizes, such that $0<\eta_{t+1}\le\eta_t,\forall t\in[T]\triangleq\{1,\dots,T\}$
\State \textbf{Initialize:} Global model $x_1 \in \mathcal{X}$; set $x_{1,m}=x_1$ for all $m \in [M]$

\For{$t = 1$ to $T$}
    \For{each client $m \in [M]$ \textbf{in parallel}}
        \State Play $x_{t,m}$; adversary picks distribution $D_{t,m}$ \label{algo:step_action}
        \State Observe $\xi_{t,m} \sim D_{t,m}$, incur loss $f(x_{t,m}, \xi_{t,m})$ \label{algo:step_loss}
        \State Update: $x_{t+1,m} = \operatorname{Proj}_{\mathcal{X}}\big(x_{t,m} - \eta_t\nabla f(x_{t,m},\xi_{t,m})\big)$ \label{algo:step_gradient}
        
        \If{$(t-1) \pmod{\tau} = 0$}
            \State Send $x_{t+1,m}$ to server \label{algo:step_avg1}
            \State Server computes $x_{t+1} = \frac{1}{M}\sum_{j=1}^M x_{t+1,j}$
            \State Receive $x_{t+1}$ and update $x_{t+1,m} \gets x_{t+1}$ \label{algo:step_avg3}
        \EndIf
    \EndFor
\EndFor

\State \textbf{Output:} Final global model $x_{T+1}$
\end{algorithmic}
\end{algorithm}

\paragraph*{Spatial and Temporal Heterogeneity.}
In addition to communication constraints, a key challenge in online federated learning is managing data heterogeneity from two distinct sources: \textit{spatial} and \textit{temporal}. 

Spatial heterogeneity arises from diverse data distributions across clients $\{D_{t,m}\}_{m=1}^M$ at any time step $t$, a well-recognized challenge in existing (batch) FL works \citep{khaled_tighter_theory_of_LSGD,SCAFFOLD}.
We quantify the spatial heterogeneity at a time step $t$ by taking the maximum deviation over the domain $\mathcal X$ between the expected local gradient and the expected global gradient:
\begin{equation}\label{spatial heterogeneity at t}
    \zeta_t^2 \triangleq \max_{x \in \mathcal X} \frac{1}{M}\sum_{m=1}^M\|\nabla f_{t,m}(x)-\nabla f_{t}(x)\|^2.
\end{equation}
On the other hand, temporal heterogeneity arises from the online nature of the problem, as the adversary shifts the underlying data distribution over time. Let $x_t^* \triangleq \arg\min_{x \in \mathcal{X}} f_t(x)$ denote the optimal decision at time step $t$ and let $x^* = \arg\min_{x \in \mathcal{X}} \sum_{t=1}^T f_t(x)$ be the best decision in hindsight over the time horizon $T$. We quantify the temporal heterogeneity at time step $t$ as the gap:
\begin{align}\label{temporal heterogeneity at t}
    K_t^2 \triangleq f_t(x^*)-f_t(x_t^*).
\end{align}
We also define the time-averaged spatial and temporal heterogeneities over the horizon $T$ as:
\begin{align}\label{avg heterogenity}
    \bar{\zeta}^2=\frac1T\sum_{t=1}^T\zeta_t^2 \quad \text{and} \quad \bar{K}^2=\frac1T\sum_{t=1}^TK_t^2.
\end{align}
Next, we discuss the theoretical regret guarantees for our proposed \algoOFL{} algorithm.

\section{Theoretical Results}
\subsection{Assumptions} First we state the assumptions used in our analysis. 

\begin{assumption}[Bounded Domain]
\label{assumption on feasible set}
The feasible set \(\mathcal{X}\subseteq\mathbb{R}^d\) is non-empty, closed, bounded, convex set.
\end{assumption}

\begin{assumption}[Unbiased with Bounded Variance]
\label{unbiased_grad, and bounded varience}
At each time $t$, for each client $m$, the stochastic gradient
is an unbiased estimator of true gradient. In addition, the stochastic gradients have a bounded variance, 
i.e., for all $x$,
\begin{align}
    & \E_{\xi\sim D_{t,m}}[\nabla f(x,\xi)] = \nabla f_{t,m}(x), \label{eq:unbiased} \\
    & \E_{\xi\sim D_{t,m}} \left[ \left\| \nabla f(x,\xi) - \nabla f_{t,m}(x) \right\|^2 \right] \leq \sigma_{t,m}^2 <\infty. \label{eq:bounded_var}
\end{align}
\end{assumption}
Based on \eqref{eq:bounded_var}, we define \(\sigma_t^2=\frac{1}{M}\sum_{m=1}^M\sigma_{t,m}^2\), and \(\bar{\sigma}^2=\frac1T\sum_{t=1}^T \sigma_t^2\).

\begin{assumption}[Smoothness]
\label{smooth}
For all \(t\in[T],m\in[M]\), \(f_{t,m}(x)\) are \(L\)-smooth. That is for all \(x,y\in\mathcal{X}\),
\begin{align*}
    f_{t,m}(y) \le f_{t,m}(x)+\langle\nabla f_{t,m}(x),y-x\rangle+\frac{L}{2}\|y-x\|^2.
\end{align*}
\end{assumption}

\begin{assumption}[Convexity]
\label{convexity}
 The expected losses \(f_{t,m}(x)\) are differentiable and convex for all \(t\in[T], m\in[M]\), i.e.,
 \begin{align*}
     f_{t,m}(y)\geq f_{t,m}(x)+\langle\nabla f_{t,m}(x), y-x\rangle, \text{ for all }x, y\in\mathcal{X}.
\end{align*}
\end{assumption}

\begin{assumption}[Strong Convexity]
\label{strongly convex assumption}
The expected losses \(f_{t,m}(x)\) are differentiable and $\mu$-strongly convex for all \(t\in[T], m\in[M]\), i.e.,
\begin{align*}
    f_{t,m}(y)\geq f_{t,m}(x)+\langle \nabla f_{t,m}(x),y-x\rangle+\frac{\mu}{2}\|y-x\|^2,
\end{align*}
for all \(x,y\in \mathcal{X}\).
\end{assumption}

Next, we present the theoretical guarantees of our proposed algorithm. For simplicity of analysis, we assume that the constraint set $\mathcal{X}$ is large enough such that the iterates of \algoOFL{} always lie inside \(\mathcal{X}\). Therefore, we can avoid the projection operator in our analysis. Note that this is done so that we can focus on the factors that enable achieving parallelization across clients in online learning. Retaining parallelization in presence of projection involves significantly more challenging analytical tools even in the offline setting \citep{yuan2021federated}. We leave incorporating projections into our approach as future work.

We establish regret bounds for two classes of loss functions: 1) smooth and convex, and 2) smooth and strongly convex. In both cases, \algoOFL{} achieves sublinear regret. In addition, we identify regimes where \algoOFL{} achieves the benefit of parallelization.

\subsection{Smooth and Convex Case}
\begin{theorem}\label{theorem 1}
Suppose Assumptions \ref{assumption on feasible set}, \ref{unbiased_grad, and bounded varience}, \ref{smooth}, and \ref{convexity} hold, and $\|x_1 - x^*\|^2 \leq D < \infty$.
Then, with constant step-size $\eta_t \equiv \eta$, the expected regret of \algoOFL{} (Algorithm \ref{O_FedAvg}) satisfies
\begin{align}\label{regret bound in thm 1}
    \E[R_T] \le \frac{D^2}{\eta}+T\eta\left[\frac{\bar{\sigma}^2}{M} + L\bar{K}^2\right]
    +\underbrace{T \eta^2L(\tau-1) \left[ \bar{\sigma}^2 + (\tau-1) \bar{\zeta}^2 \right]}_{\substack{\text{Drift due to spatial heterogeneity} \\ \text{(Higher order)}}} + \underbrace{T \eta^2 L^2 (\tau-1)^2\bar{K}^2}_{\substack{\text{Drift due to temporal } \text{heterogeneity}}},
\end{align}
where we omit numerical constants for simplicity.
With an appropriate choice of the step-size $\eta$,
\begin{align}
    \E[R_T] \le \mathcal O \left( D \sqrt{T \left[ \frac{\bar{\sigma}^2}{M} + L \bar{K}^2 \right]} \right) + \mathcal O \left( M \tau^2 \right). \label{eq:thm1_2}
\end{align}
\end{theorem}
\begin{remark}[\textbf{Generalizes existing results}]
In the absence of time-varying data distributions, i.e., $D_{t,m} \equiv D_m$, for all $t\in [T]$ and $m \in [M]$, Theorem \ref{theorem 1} recovers the convergence results for federated stochastic optimization \citep{khaled_tighter_theory_of_LSGD, koloskova2020unified}.
\end{remark}

\begin{remark}[\textbf{Parallelization}]
\label{rem:linear_speedup}
Note that the drift terms in \eqref{regret bound in thm 1} scale with $\eta^2$, and are hence, negligible for small enough $\eta$. However, a key insight from \eqref{regret bound in thm 1} lies in the term \(T\eta\left[\frac{\bar{\sigma}^2}{M}+L\bar{K}^2\right]\). The periodic averaging of models across clients in \algoOFL{} introduces the variance reduction factor $1/M$ in the variance term $\bar{\sigma}^2$. However,  the temporal heterogeneity term \(\bar{K}^2\) does not improve on averaging. Nonetheless, if the average temporal variation $\bar{K}^2$ is mild enough, such that \(\frac{\bar{\sigma}^2}{M} \geq L\bar{K}^2\), then the network can achieve an improved expected regret $\mathcal{O}(\sqrt{T/M})$. In \citep{patel_OFL_with_bandit_feedback}, the authors proved the negative result that when solving OCO problems for smooth, Lipschitz functions, one cannot get any benefit of parallelization in the worst case. However, using the SEA framework, we have identified a \textit{benign} regime where the benefit of parallelization can be achieved in the online setting. 

Intuitively, if the data distributions do not change much over time ($\Rightarrow \bar{K}^2$ is small), stochastic gradient variance can mask some of the temporal variation in gradients.  
\end{remark}

\begin{remark}[\textbf{Dependence on $\tau$ and Communication Cost}]
The drift terms in the regret bounds in \eqref{regret bound in thm 1} grow quadratically with the synchronization period \(\tau\). Using \eqref{eq:thm1_2} and under the benign regime discussed in Remark \ref{rem:linear_speedup}, the regret bound reduces to 
\begin{equation*}
    \mathcal{O}(\sqrt{T/M}) + \mathcal{O}(M \tau^2).
\end{equation*}
To retain $\mathcal{O}(\sqrt{T/M})$ regret, $\tau$ needs to satisfy $\tau = \mathcal{O}\left( \frac{T^{1/4}}{M^{3/4}} \right)$. Consequently, for large enough $T$, we can communicate for as little as $T/\tau = \mathcal{O}((MT)^{3/4})$ of the total $T$ iterations, without affecting the expected regret bound. Similar results were shown in \citep{khaled_tighter_theory_of_LSGD} for the federated stochastic optimization problem. We achieve similar communication savings in a more general online learning problem.
\end{remark}

\subsection{Smooth and Strongly Convex Case}
\begin{theorem}\label{theorem 2}
Suppose Assumptions \ref{assumption on feasible set}, \ref{unbiased_grad, and bounded varience}, \ref{smooth}, and \ref{strongly convex assumption} hold. Then, using appropriately scaled decaying step-sizes $\{\eta_t\}$ in \algoOFL{} (Algorithm \ref{O_FedAvg}), the expected regret satisfies
\begin{align}\label{regret bound in theorem 2}
    \E[R_T] \le \mathcal{O} \left( \left[ \frac{\sigma_{\text{max}}^2}{M}+L K_{\text{max}}^2 \right] \frac{1}{\mu} (1 + \log T) \right)
    + \underbrace{\mathcal{O}\left(\frac{L^2 \tau}{\mu^3}\Big(\sigma_{\text{max}}^2 + \tau\zeta_{\text{max}}^2 + \tau LK_{\text{max}}^2\Big)\right)}_{\text{Drift due to spatial and temporal heterogeneity}} + \underbrace{E_{1:t_0-1}}_{\substack{\text{Sum of } \Theta(\kappa^{3/2} \tau) \text{terms}}},
\end{align}
where \(\sigma_{\text{max}}^2=\max_{t\in[T]} \sigma_t^2 \), \(\zeta_{\text{max}}^2=\max_{t\in[T]} \zeta_t^2 \), and \(K_{\text{max}}^2=\max_{t\in[T]} K_t^2 \), and $E_{1:t_0-1}$ represents the accumulated error from \(t=1\) to \(t=t_0-1\), defined as
\begin{align*}
    E_{1:t_0-1} = \mathcal{O}\left( \sum_{t=1}^{t_0-1} \left( \frac{L}{\mu t} + \frac{L^3(\tau-1)^2}{\mu^3t^2} \right) \E[f_t(x_t)-f_t(x^*)] \right),
\end{align*}
where $t_0 \triangleq \inf \{t: \frac{4L}{\mu t} + \frac{288L^3(\tau-1)^2}{\mu^3t^2} < \frac12 \} = \Theta(\kappa^{3/2} \tau)$, for $\kappa = L/\mu$.
\end{theorem}

\begin{remark}[\textbf{Parallelization}]
For large enough $T$, the first term in \eqref{regret bound in theorem 2} is the dominant term. Similar to Theorem \ref{theorem 1}, the stochastic gradient variance term \(\sigma_{\text{max}}^2\) has a variance reduction factor \(\frac1M\) due to periodic averaging. Similar to the discussion in the smooth convex case, if the \textit{worst case} temporal heterogeneity $K_{\text{max}}^2$ is mild enough such that $\frac{\sigma_{\text{max}}^2}{M} \geq  LK_{\text{max}}^2$, then we can achieve an improved expected regret $\mathcal{O}(\log T/M)$. Once again, we improve the existing results for online federated learning \citep{OFL_mitra} by using the SEA formulation.

Also, note that in Theorem \ref{theorem 2}, the regret bound depends on the worst-case quantities $\sigma^2_{\text{max}},~ \zeta^2_{\text{max}},~ K^2_{\text{max}}$, compared to the average-case dependence in Theorem \ref{theorem 1}. This is due to the decaying step-size used in Theorem \ref{theorem 2}.
\end{remark}

\begin{remark}[\textbf{Effect of spatial vs temporal heterogeneity}]
In both the convex and strongly convex settings, the dominant regret term depends on the temporal heterogeneity ($\bar{K}^2$ in the $\mathcal{O}(\eta T)$ term in Theorem \ref{theorem 1}, and $K_{\text{max}}^2$ in the $\log T$ term in Theorem \ref{theorem 2}). On the other hand, spatial heterogeneity ($\bar{\zeta}^2$ in Theorem \ref{theorem 1} and $\zeta_{\text{max}}^2$ in Theorem \ref{theorem 2}) only affects the higher-order drift terms or additive constants. This implies that temporal shifts in the data distribution have a greater effect on the regret guarantees than the spatial variations in data across clients.
\end{remark}

\begin{remark}[\textbf{Impact of accumulated error ($E_{1:t_0-1}$)}]
The term $E_{1:t_0-1}$ represents the accumulated error during the initial phase of the algorithm up to a fixed time step $t_0$. While $t_0$ is constant and this sum contains a finite number of terms, its magnitude depends on $f_t(x_t) - f_t(x^*)$. Assuming $\mathcal{X}$ has a diameter $D<\infty$, \(E_{1:t_0-1}\) can be bounded in terms of $D$ using the fact that,
\begin{align*}
f_t(x_t)-f_t(x^*) &= f_t(x_t)-f_t(x_t^*) + \underbrace{f_t(x_t^*)-f_t(x^*)}_{\leq 0} \\
    & \leq \frac{L \|x_t-x_t^*\|^2}{2} \leq \frac{LD^2}{2} 
\end{align*}
 Consequently, if $D$ is a function of $T$ such that $E_{1:t_0-1}$ grows asymptotically faster than $\mathcal{O}(\log T)$, the initial penalty will dominate the bound and completely destroy the logarithmic regret guarantee. Therefore, assuming a bounded domain diameter that is strictly independent of $T$ is necessary.
\end{remark}

\section{Analysis and Proof Sketch}
In \algoOFL{}, the local update at client $m$ is given by
\begin{equation*}
    x_{t+1,m} = x_{t,m} - \eta_t \nabla f(x_{t,m}, \xi_{t,m}).
\end{equation*}
At each iteration $t$, we define the \textit{virtual} global average iterate $x_t \triangleq \frac{1}{M} \sum_{m=1}^M x_{t,m}$, with the updates given as
\begin{equation}\label{virtual itr}
    x_{t+1} = x_{t} - \frac{\eta_{t}}{M} \sum_{m=1}^M \nabla f(x_{t,m},\xi_{t,m}).
\end{equation}
In \algoOFL{}, an explicitly computed global model exists only at synchronization steps. However, for the sake of analysis, we track a \textit{virtual} global average iterate $x_t$ at every time step. Accordingly, we define the consensus error \(V_t\) to measure the distance between the local models $\{x_{t,m}\}_{m}$ and the virtual global model $x_t$ at time \(t\).
\begin{equation}\label{def of V_t}
    V_t \triangleq \frac{1}{M}\sum_{m=1}^M\|x_{t,m}-x_t\|^2.
\end{equation}
For simplicity of notation, we denote \(\E_t\) for conditional expectation conditioning on \(x_t\), and define
$\Gtm \triangleq \nabla f(x_{t,m},\xi_{t,m}),\text{ and }~\vGtm\triangleq \nabla f(x_t,\xi_{t,m})$.
Furthermore, we frequently omit the limits on summations for brevity. First, we state some intermediate results that will be used to prove both Theorem \ref{theorem 1} and \ref{theorem 2}. Subsequently, we discuss the outline of the individual proofs.

\subsection{Intermediate Results}
\begin{lemma}[Regret decomposition]\label{lemma: regret decomposition}
If the expected global losses $\{f_t\}$ for all $t$ satisfy Assumption \ref{smooth}, then
\begin{equation}\label{regret decomposition eq}
\begin{aligned}
    \E[R_T] & \le \underbrace{\sum_{t}\E [f_t(x_t)-f_t(x^*)]}_{\substack{\text{``Virtual'' Centralized Regret}}}+\frac L2 \underbrace{\sum_{t}\E[V_t]}_{\substack{\text{Consensus Error}}}
    \end{aligned}
\end{equation}
\end{lemma}

\begin{proof}
Using \eqref{regret equ}, we can decompose the component terms of $R_T$ at each time $t$ to get

 \begin{align*}
 \begin{aligned}
    \frac1M\sum_{m}f_t(x_{t,m})-f_t(x^*)&=\frac1M\sum_m[f_t(x_{t,m})-f_{t}(x_t)
    +f_{t}(x_t)-f_t(x^*)]\\
    &\le\frac1M\sum_m\left\langle\nabla f_t(x_t),x_{t,m}-x_t\right\rangle+\frac L{2M}\sum_m\|x_{t,m}-x_t\|^2
    +f_t(x_t)-f_t(x^*) \\
    &=f_t(x_{t})-f_t(x^*)+\frac L2V_t,
\end{aligned}
\end{align*}
where the inequality follows from Assumption \ref{smooth}, and the final equality holds because \(x_{t}=\frac1M\sum_mx_{t,m}\). Summing over $t$, we get the desired bound.
\end{proof}
The expected regret in Lemma \ref{lemma: regret decomposition} is bounded in terms of two terms. First, the regret of a ``virtual'' average model sequence $\{ x_t \}_{t\in[T]}$, and second, the consensus error that quantifies the deviation of the individual client models $\{x_{t,m}\}_{m \in [M]}$ in \algoOFL{} (Algorithm \ref{O_FedAvg}) from the virtual average model. 
In the next result, we bound the progress made between consecutive virtual iterates in \eqref{virtual itr}.

\begin{lemma}[Progress per iteration]
\label{lemma:gradient_bound}
Under Assumptions \ref{unbiased_grad, and bounded varience}, \ref{smooth}, the following bound holds
\begin{equation}
    \begin{aligned}
        \mathbb{E}_t\Big\|\frac{1}{M}\sum_{m}\Gtm\Big\|^2 & \le \frac{10}{M} \sigma_t^2 + 2L^2V_t + \underbrace{4L\big(f_t(x_t) - f_t(x_t^*)\big)}_{\substack{\text{Instantaneous Online} \\ \text{Learning Error}}} \notag
    \end{aligned}
\end{equation}
where $x_t^* \triangleq \arg\min_x f_t(x)$ is the optimum at time $t$.
\end{lemma}

\begin{proof}
\begin{equation}\label{1st eq of lemma 2}
    \begin{aligned}
    \mathbb{E}_t\Big\|\frac{1}{M}\sum_m\Gtm\Big\|^2\le\frac{2}{M^2}\mathbb{E}_t\Big\|\sum_m\vGtm\Big\|^2 + \frac{2}{M^2}\mathbb{E}_t\Big\|\sum_m\big(\Gtm - \vGtm\big)\Big\|^2
    \end{aligned}
\end{equation}
The second term in \eqref{1st eq of lemma 2} can be bounded as follows:
\begin{align*}
    & \frac{2}{M^2} \mathbb{E}_t \Big\| \sum_m \big(\Gtm \pm \nabla f_{t,m}(x_{t,m}) \pm \nabla f_{t,m}(x_t) - \vGtm \big)\Big\|^2 \\
    &\le\frac4{M^2}\sum_m\E_t\| \Gtm-\nabla f_{t,m}(x_{t,m})\|^2\\
    &\quad+\frac4{M^2}\sum_{m}\E_t\|\vGtm-\nabla f_{t,m}(x_t)\|^2 + \frac{2L^2}M\sum_m\|x_{t,m}-x_t\|^2\\
    &\le \frac8{M^2}\sum_m\sigma_{t,m}^2+2L^2V_t. \tag{Using Assumption \ref{unbiased_grad, and bounded varience}}
\end{align*}
To bound the first term in \eqref{1st eq of lemma 2},
\begin{align*}
\begin{aligned}
\frac{2}{M^2}\mathbb{E}_t\Big\|\sum_m\vGtm\Big\|^2
&\le\frac{2}{M^2}\sum_m\sigma_{t,m}^2 + 2\left\|\nabla f_{t}(x_t)\right\|^2\\
&\le \frac2{M^2}\sum_m\sigma_{t,m}^2+4L\left(f_t(x_t)-f_t(x_t^*)\right).
\end{aligned}
\end{align*}
The last inequality follows from the smoothness of \(f_{t}\). Substituting the two bounds above in \eqref{1st eq of lemma 2}, yields desired result.
\end{proof}
The first term in the bound in Lemma \ref{lemma:gradient_bound} shows the variance reduction (the $1/M$ factor) due to averaging independently sampled stochastic gradients in \eqref{virtual itr}.  
The consensus error in the second term accounts for the fact that the local stochastic gradients are computed at different local models. Finally, the third term reflects the instantaneous online learning error owing to the lack of prior knowledge of the true data distributions $\{D_{t,m} \}_{m \in [M]}$ at time $t$.
In the following lemma, we bound the client drift term.

\begin{lemma}[Consensus Error/Client Drift]
\label{lemma:bound of V_t}
Under Assumptions \ref{unbiased_grad, and bounded varience}, \ref{smooth}, \ref{convexity}, and step size condition $\eta \le \frac{1}{4L(\tau-1)}$, the sum of the expected consensus error bounded by:
\begin{align}
    \sum_{t}\E[V_t]\le4\eta^2(\tau-1)\sum_{t} \Big(\sigma_{t}^2 + \underbrace{3 (\tau-1)\zeta_t^2}_{\substack{\text{Effect of ``Spatial''} \\\text{Heterogeneity}}} \Big) 
    + 12 \eta^2 (\tau-1)^2 L \sum_t \E[f_t(x_t)-f_t(x_{t}^*)].
\end{align}
\end{lemma}
\begin{proof}
    If $t-1 \pmod \tau = 0$, then $V_t = 0$.
Let $\tau k+1 < t < \tau(k+1)+1$, $k \in \mathbb{N} \cup \{0\}$. Then $V_t \ne 0$, and from \eqref{virtual itr}
\begin{equation*}
\begin{aligned}
    \Big\|x_{t,m} - x_t\Big\|^2 &= \eta^2 \Big\| \sum_{i=\tau k+1}^{t-1} \Big( \frac{1}{M}\sum_{l=1}^M\Gil- \Gim \Big) \Big\|^2\\
    &=2\eta^2\Big\|\sum_{i=\tau k+1}^{t-1}\Big(\frac1M\sum_{l=1}^M\Gil-\nabla f_{i,l}(x_{i,l})-(\Gim-\nabla f_{i,m}(x_{i,m}))\Big)\Big\|^2\\
    &\quad+2\eta^2\Big\|\sum_{i=\tau k+1}^{t-1}\Big(\frac1M\sum_{l=1}^M\nabla f_{i,l}(x_{i,l})-\nabla f_{i,m}(x_{i,m})\Big)\Big\|^2.
\end{aligned}
\end{equation*}
Taking the expectation on both sides, we get:
\begin{align*}
\begin{aligned}
    \E\|x_{t,m}-x_t\|^2&\le2\eta^2\sum_{i=\tau k+1}^{t-1}\E\Big\|\Gim-\nabla f_{i,m}(x_{i,m})\Big\|^2+2\eta^2(\tau-1)\sum_{i=\tau k+1}^{t-1}\E\Big\|\nabla f_{i,m}(x_{i,m})\Big\|^2
    \end{aligned}
\end{align*}
Averaging over M yields:
\begin{equation}\label{E[V_t] bound}
\begin{aligned}
    \E[V_t]&\le\frac{2\eta^2}{M}\sum_{i=\tau k+1}^{t-1}\sum_{m} \sigma_{i,m}^2+\frac{2\eta^2(\tau-1)}{M}\sum_{i=\tau k+1}^{t-1}\sum_m \E\|\nabla f_{i,m}(x_{i,m})\|^2
\end{aligned}
\end{equation}
Applying Assumption \ref{unbiased_grad, and bounded varience} along with \eqref{spatial heterogeneity at t} and \eqref{def of V_t}, we obtain
\begin{align*}
\begin{aligned}
    \frac{1}{M}\sum_m \E\|\nabla f_{i,m}(x_{i,m})\|^2&\le3L^2\E[V_i]+3\zeta_i^2+6L\E[f_i(x_i)-f_i(x_{i}^*)]
    \end{aligned}
\end{align*}
Substituting the above bound in \eqref{E[V_t] bound} and summing over the interval \(\tau k+2\le t\le \tau (k+1)\) yields:
\begin{align*}
    \sum_{t=\tau k+2}^{\tau (k+1)}\E[V_t]
    \le 2\eta^2(\tau-1)\sum_{t=\tau k+2}^{\tau (k+1)}\sigma_{t}^2+ 6\eta^2(\tau - 1)^2 \sum_{t=\tau k+2}^{\tau (k+1)} \Big(L^2\E[V_t] + \zeta_t^2 + 2L [f_t(x_t)-f_t(x_{t}^*)]\Big)
\end{align*}
By rearranging terms and using $\eta \le \frac{1}{4L(\tau-1)}$, we obtain
\begin{equation*}
\begin{aligned}
    \sum_{t=\tau k+2}^{\tau (k+1)}\E[V_t]\le4\eta^2(\tau-1)\sum_{t=\tau k+2}^{\tau (k+1)}\sigma_{t}^2+12\eta^2(\tau-1)^2\sum_{t=\tau k+2}^{\tau (k+1)}\Big(\zeta_t^2+2L\E[f_t(x_t)- f_t(x_{t}^*)]\Big)
    \end{aligned}
\end{equation*}
Finally, summing over the \(T\) rounds gives the final bound. 
\end{proof}

The upper bound in Lemma \ref{lemma:bound of V_t} consists of three terms. While the first and last terms are already discussed in Lemma \ref{lemma:gradient_bound}, the middle term explicitly captures the \textit{spatial} heterogeneity across the network. This term arises because at each time $t$, the clients encounter different local data distributions.

The bounds in Lemmas \ref{lemma:gradient_bound} and \ref{lemma:bound of V_t} contain the instantaneous online learning error $f_t(x_t)-f_t(x_{t}^*)$, which can be further decomposed to yield the ``virtual'' centralized regret $f_t(x_t)-f_t(x^*)$ and the temporal heterogeneity $K_t^2$ as follows
\begin{equation}
\begin{aligned}\label{split of moving target term}
\sum_t\E[f_t(x_t)-f_t(x_{t}^*)] = \sum_{t}\E[f_t(x_t)-f_t(x^*)] + \sum_t\ \underbrace{f_t(x^*)-f_t(x_t^*)}_{\triangleq K_t^2} .
\end{aligned}
\end{equation}
\subsection{Proof of Theorem \ref{theorem 1} (Smooth Convex Case)}
First, we bound the ``virtual'' regret in Lemma \ref{lemma: regret decomposition}.
\begin{lemma}\label{lemma:regret bound of x_t}
Suppose Assumptions  \ref{unbiased_grad, and bounded varience}, \ref{smooth}, and \ref{convexity} hold. If the learning rate satisfies $\eta \le \min\left\{\frac{1}{8L}, \frac{1}{4\sqrt{6}L(\tau-1)}\right\}$, then the following bound holds.
\begin{equation}
\begin{aligned}
    \sum_t \mathbb{E}[f_t(x_t) - f_t(x^*)] &\le \frac{\|x_{1}-x^*\|^2}{\eta}+\left(\frac{10\eta}{M} +8\eta^2L(\tau-1)\right) \sum_t\sigma_t^2+ 24\eta^2 L(\tau-1)^2 \sum_t\zeta_t^2\\
    &\quad+\left(4\eta L+24\eta^2L^2(\tau-1)^2\right)\sum_t [f_t(x^*)-f_t(x_t^*)].
\end{aligned}
\end{equation}
\end{lemma}

\begin{proof}
Using the virtual iterate update in \eqref{virtual itr}, we get
\begin{equation}\label{square distance to x^*}
\begin{aligned}
    \mathbb{E}_t\left\|x_{t+1}-x^*\right\|^2& = \left\|x_t-x^*\right\|^2 + \eta^2\mathbb{E}_t\Big\|\frac{1}{M}\sum_m\Gtm\Big\|^2- \frac{2\eta}{M}\sum_m\left\langle\nabla f_{t,m}(x_{t,m}), x_t-x^*\right\rangle
\end{aligned}
\end{equation}
The second term above is bounded using Lemma~\ref{lemma:gradient_bound}, while the inner product term is bounded using Assumptions \ref{smooth} and \ref{convexity}. Taking the total expectation and summing over the time horizon $T$, we obtain
\begin{align*}
\begin{aligned}
    \E\|x_{T+1}-x^*\|^2&\le \|x_1-x^*\|^2
    -2\eta\sum_{t}\E[f_t(x_t)-f_t(x^*)]+4\eta^2L\sum_t\E[f_t(x_t)-f_t(x_t^*)]\\
    &\quad+\eta L(1+2\eta L)\sum_{t}\mathbb{E}[V_t]+\frac{10\eta^2}{M^2}\sum_{t,m}\sigma_{t,m}^2
    \end{aligned}
\end{align*}
Applying Lemma \ref{lemma:bound of V_t} to bound the consensus error term yields:
\begin{align*}
    \begin{aligned}
        \sum_{t}\mathbb{E}[f_t(x_t)-f_t(x^*)]&\le\frac{\|x_{1}-x^*\|^2}{2\eta}+\left(\frac{5\eta}{M}
        +2\eta^2L(1+2\eta L)(\tau-1)\right)
        \sum_{t}\sigma_t^2+ 6\eta^2L(1+2\eta L)(\tau-1)^2 
        \sum_{t}\zeta_t^2\\
        &\quad+\left(2\eta L+12\eta^2L^2(1+2\eta L)(\tau-1)^2\right)\sum_t\E[f_t(x_t)-f_t(x_t^*)]
    \end{aligned}
\end{align*}
Next, we split the last term as in \eqref{split of moving target term}, and rearrange it accordingly.
\begin{equation*}
\begin{aligned}
    & \left(1-2\eta L-12\eta^2L^2(1+2\eta L)(\tau-1)^2\right)\sum_{t}\E[f_t(x_t)-f_t(x^*)] \\
    &\le\frac{\|x_{1}-x^*\|^2}{2\eta}
    +\left(\frac{5\eta}{M}
    +2\eta^2L(1+2\eta L)(\tau-1)\right)
    \sum_{t}\sigma_t^2 \\
    &+6\eta^2L(1+2\eta L)(\tau-1)^2
    \sum_{t}\zeta_t^2\\
    &+\left(2\eta L+12\eta^2L^2(1+2\eta L)(\tau-1)^2\right)\sum_t K_t^2
\end{aligned}
\end{equation*}
Using the conditions $\eta \le \min\Big\{\frac{1}{8L}, \frac{1}{4\sqrt{6}L(\tau-1)}\Big\}$, and rearranging the terms, we complete the proof.
\end{proof}

Substituting the bounds in Lemmas \ref{lemma:bound of V_t}, \ref{lemma:regret bound of x_t} into Lemma \ref{lemma: regret decomposition}, and then using \eqref{split of moving target term}
concludes the proof of Theorem \ref{theorem 1}.

\subsection{Proof of Theorem \ref{theorem 2} (Smooth Strongly-convex Case)}
To prove Theorem \ref{theorem 2}, we need the following lemmas. In the first lemma, we bound the consensus error. The result is identical to Lemma \ref{lemma:bound of V_t}, with the only difference of using time-varying step-sizes here. 
\begin{lemma}\label{lemma:bound of V_t with eta_t}
Under Assumptions \ref{unbiased_grad, and bounded varience}, \ref{smooth}, \ref{strongly convex assumption}, and step size condition $\eta_t \le \frac{1}{4L(\tau-1)}$, the iterates of Algorithm \ref{O_FedAvg} satisfy
\begin{equation}
\begin{aligned}
    \sum_t \E[V_t]\le4(\tau-1)\sum_t\eta_t^2\sigma_{t}^2+12(\tau-1)^2\sum_t\eta_t^2\Big(\zeta_t^2+L\E[f_t(x_t)-f_t(x_{t}^*)]\Big)
\end{aligned}
\end{equation}
\end{lemma}
The next result establishes a regret bound for the virtual iterates in the case of strongly convex loss functions.

\begin{lemma}\label{lemma:optimization error with strong convex}
Suppose Assumptions \ref{unbiased_grad, and bounded varience}, \ref{smooth}, \ref{strongly convex assumption} hold. 
Then, the regret bound with respect to the virtual iterates $\{ x_t \}$ is given as
\begin{equation}\label{opt error for strongly convex loss}
\begin{aligned}
    \sum_t\E[f_t(x_t)-f_t(x^*)] 
    &\le \left(\frac{20\sigma_{\text{max}}^2}{M\mu}+\frac{8LK_{\text{max}}^2}{\mu}\right)\log T\\
    &\qquad +\mathcal{O}\left( \frac{\sigma_{\text{max}}^2}{M\mu} + \frac{LK_{\text{max}}^2}{\mu} + \frac{L^2 \tau}{\mu^3}\Big(\sigma_{\text{max}}^2 + \tau\zeta_{\text{max}}^2 + \tau L K_{\text{max}}^2\Big) \right)+ E_{1:t_0-1},
\end{aligned}
\end{equation}
where \(\sigma_{\text{max}}^2=\max_{t\in[T]} \sigma_t^2 \), \(\zeta_{\text{max}}^2=\max_{t\in[T]} \zeta_t^2 \), and \(K_{\text{max}}^2=\max_{t\in[T]} K_t^2 \), and
\begin{align*}
    E_{1:t_0-1} = \mathcal{O}\left( \sum_{t=1}^{t_0-1} \left( \frac{L}{\mu t} + \frac{L^3(\tau-1)^2}{\mu^3t^2} \right) \E[f_t(x_t)-f_t(x^*)] \right),
\end{align*}
where $t_0 \triangleq \inf \big\{t: \frac{4L}{\mu t} + \frac{288L^3(\tau-1)^2}{\mu^3t^2} \leq \frac{1}{2} \big\}$.
\end{lemma}

\begin{proof}
    To bound the last term of \eqref{square distance to x^*}, we apply Assumption~\ref{strongly convex assumption} together with Young's inequality, and obtain
\begin{align*}
\begin{aligned}
    - \frac{2\eta_t}{M} \sum_m \langle \nabla f_{t,m}(x_{t,m}), x_t - x^* \rangle \le \eta_t (\beta - \mu) \| x_t - x^* \|^2 + \frac{L^2 \eta_t}{\beta} V_t
    - \frac{2\eta_t}{M} \sum_m \left( f_{t,m}(x_t) - f_{t,m}(x^*) \right)
    \end{aligned}
\end{align*}
Invoking the above bound in \eqref{square distance to x^*} with $\beta = \frac{\mu}{2}$, and taking expectation of both sides, we get
\begin{equation}\label{eq:one_step_regret}
\begin{aligned}
    \frac{1}{M} \sum_m \mathbb{E}\left[ f_{t,m}(x_t) - f_{t,m}(x^*) \right]&\le \frac{\left(1 - \frac{\eta_t \mu}{2}\right)}{2\eta_t} \E\| x_t - x^* \|^2- \frac{1}{2\eta_t} \mathbb{E} \| x_{t+1} - x^* \|^2\\
    &\quad+ \frac{L^2}{\mu} \E [V_t] + \frac{\eta_t}{2} \mathbb{E} \Big\| \frac{1}{M} \sum_{m}\Gtm  \Big\|^2
    \end{aligned}
\end{equation}
To induce a telescoping sum, we set the decaying step size $\eta_t = \frac{2}{\mu t}$, and sum \eqref{eq:one_step_regret} over $t$. We drop the final negative term $-T\mathbb{E}\|x_{T+1} - x^*\|^2$ to get
\begin{align*}
    \sum_{t} \frac{1}{M} \sum_{m} \mathbb{E} \left[ f_{t,m}(x_t) - f_{t,m}(x^*) \right] 
    &= \sum_{t} \mathbb{E} \left[ f_{t}(x_t) - f_{t}(x^*) \right] \\
    &\le \frac{L^2}{\mu} \sum_{t} \mathbb{E}[V_t]+ \frac{1}{\mu} \sum_{t} \frac{1}{t} \mathbb{E} \Big\| \frac{1}{M} \sum_{m}\Gtm  \Big\|^2\\
    &\le \sum_{t}\frac{L^2}{\mu}\left(1+\frac2t\right)\E[V_t]+\frac{10}{M\mu}\sum_{t}\frac{\sigma_t^2}{t}+\frac{4L}{\mu}\sum_{t}\frac{\E[f_t(x_t)-f_t(x_t^*)]}t \tag{From Lemma \ref{lemma:gradient_bound}}
\end{align*}
Since \(1+2/t\le3\) for all \(t\ge1\), we can upper bound the coefficient of \(\E[V_t]\), by \(\frac{3L^2}{\mu}\). Next, substituting the bound in Lemma \ref{lemma:bound of V_t with eta_t} with \(\eta_t=\frac2{\mu t}\), we get
\begin{align}
    \sum_{t=1}^T\E[f_t(x_t)-f_t(x^*)] &\le \frac{1}{\mu} \sum_{t=1}^T\frac{1}{t} \left[ 4L K^2_{\text{max}} +\frac{10\sigma^2_{\text{max}}}{M} \right] \notag+\frac{48L^2(\tau-1)}{\mu^3}\sum_{t=1}^T\frac{1}{t^2} \left[ \sigma^2_{\text{max}} + (\tau-1) (3\zeta^2_{\text{max}} + 6L K^2_{\text{max}}) \right] \notag \\
    &\quad+\sum_{t=1}^T\left(\frac{288 L^3 (\tau-1)^2}{\mu^3t^2}+\frac{4 L}{\mu t}\right)\E[f_t(x_t)-f_t(x^*)]. \label{eq:lemma_6_1}
\end{align}
To handle the last term on the right-hand side above, we introduce a threshold time step $t_0$, defined as
\begin{equation*}
    t_0 = \max \left\{ \left\lceil \frac{8L(\tau-1)^2}{\mu} \right\rceil, \left\lceil \frac{4L}{\mu} \left( 1 + \sqrt{1 + \frac{36L(\tau-1)^2}{\mu}} \right) \right\rceil \right\}.
\end{equation*}
This ensures that \(\frac{4L}{\mu t}+\frac{288L^3(\tau-1)^2}{\mu^3t^2}\le\frac12\) for all \(t\ge t_0\). We also use
\begin{equation*}
\sum_t \eta_t^2 \le \frac{4}{\mu^2} \sum_{t=1}^\infty \frac{1}{t^2} = \frac{4}{\mu^2} \frac{\pi^2}{6} = \frac{2\pi^2}{3\mu^2}.
\end{equation*}
After some rearranging, from \eqref{eq:lemma_6_1}, we get
\begin{equation*}
\begin{aligned}
    \sum_t \E[f_t(x_t)-f_t(x^*)]
    &\le \mathcal{O} \Bigg( \frac{L^2 (\tau-1)}{\mu^3} \left[ \sigma_{\text{max}}^2 + (\tau-1) (\zeta_{\text{max}}^2 + L K_{\text{max}}^2) \right] \Bigg)+ \mathcal{O} \Bigg( (1+\log T) \left[\frac{\sigma_{\text{max}}^2}{M\mu}+\frac{LK_{\text{max}}^2}{\mu} \right] \Bigg) \\
    &\quad+\sum_{t=1}^{t_0-1} \left(\frac{576 L^3 (\tau-1)^2}{\mu^3t^2}+\frac{8 L}{\mu t} - 1\right)\E[f_t(x_t)-f_t(x^*)]
    \end{aligned}
\end{equation*}
The last term represents the finite accumulated error, denoted by $E_{1:t_0-1}$. This completes the proof.
\end{proof}
\noindent
Next, we finish the proof of Theorem \ref{theorem 2}.

\begin{proof}
Substituting the bound in Lemma \ref{lemma:bound of V_t with eta_t} into \eqref{regret decomposition eq}, we get
\begin{align}\label{first eq in thm 2}
    \E[R_T]&\le \sum_t \Big( 1 + 12L^2(\tau-1)^2 \eta_t^2 \Big) \E[f_t(x_t)-f_t(x^*)] \nonumber\\
    &\quad+ \frac{L}{2} \sum_t \eta_t^2 \Big[ 4(\tau-1)\sigma_{t}^2+ 12(\tau-1)^2 \Big( \zeta_t^2 + 2L\E[f_t(x^*)-f_t(x_{t}^*)] \Big) \Big].
\end{align}
To bound the first term in \eqref{first eq in thm 2}, first note that the step size condition $\eta_t \le \frac{1}{4L(\tau-1)}$ ensures $1+12\eta_t^2L^2(\tau-1)^2 \le 2$. Second, $\sum_t \E[f_t(x_t)-f_t(x^*)]$ is bounded in Lemma \ref{lemma:optimization error with strong convex} above. Finally, over a bounded domain, $E_{1:t_0-1}$ can be bounded by a constant using smoothness (Assumption \ref{smooth}).
For the remaining terms in \eqref{first eq in thm 2}, similar to how we showed in \eqref{eq:lemma_6_1} above, 
\begin{equation}\label{eq:proof_thm2_2}
    \frac{L}{2} \sum_t \eta_t^2 \Big[ 4(\tau-1)\sigma_{t}^2 + 12(\tau-1)^2 \Big( \zeta_t^2 + 2LK_t^2 \Big) \Big] 
    \le \mathcal{O}\left( \frac{L \tau}{\mu^2} \Big( \sigma_{\text{max}}^2 + \tau \zeta_{\text{max}}^2 + L \tau K_{\text{max}}^2 \Big) \right). 
\end{equation}
Finally, substituting the bound from Lemma \ref{lemma:optimization error with strong convex} and \eqref{eq:proof_thm2_2} into \eqref{first eq in thm 2} obtains the eventual bound.
\end{proof}

\section{Conclusion}
In this work, we analyze a novel online federated learning (OFL) framework under a stochastically extended adversary (SEA) model, where the adversary chooses the data distributions observed at each client. The proposed setting captures the stochastic nature of real-world data observations as well as the spatial (across clients) and temporal (over time) variations in the data distribution. We propose the \algoOFL{} algorithm, where the clients locally update their models via online gradient descent and communicate with the server periodically. We prove that \algoOFL{} achieves a network regret bound of $\mathcal{O}(\sqrt{T})$ for smooth and convex loss functions, and $\mathcal{O}(\log T)$ for smooth and strongly convex loss functions. In particular, our regret bounds highlight the distinct effects of spatial and temporal heterogeneity. Under mild temporal variations, we provably show the benefit of parallelization in terms of improved regret bounds, hence improving the existing worst-case bounds in certain regimes. In future work, we plan to incorporate projection steps into our analysis using dual averaging-based techniques. Other potential future directions include studying the more challenging metric of dynamic regret, as well as more realistic federated settings, like partial client participation.

\newpage
\bibliographystyle{icml}
\bibliography{references}

@INPROCEEDINGS{OFL_mitra,
  author={Mitra, Aritra and Hassani, Hamed and Pappas, George J.},
  booktitle={2021 60th IEEE Conference on Decision and Control (CDC)}, 
  title={Online Federated Learning}, 
  year={2021},
  volume={},
  number={},
  doi={10.1109/CDC45484.2021.9683589}}

@inproceedings{yuan2021federated,
  title={Federated composite optimization},
  author={Yuan, Honglin and Zaheer, Manzil and Reddi, Sashank},
  booktitle={International Conference on Machine Learning},
  pages={12253--12266},
  year={2021},
  organization={PMLR}
}

@ARTICLE{ganguly_OFL_nonstationary,
author={Ganguly, Bhargav and Aggarwal, Vaneet},
journal={IEEE/ACM Transactions on Networking},
title={{Online Federated Learning via Non-Stationary Detection and Adaptation Amidst Concept Drift}},
year={2024},
volume={32},
number={01},
ISSN={1558-2566},
pages={643-653},
doi={10.1109/TNET.2023.3294366},
}

@article{sachs_SEA,
  title={Between stochastic and adversarial online convex optimization: Improved regret bounds via smoothness},
  author={Sachs, Sarah and Hadiji, H{\'e}di and van Erven, Tim and Guzm{\'a}n, Crist{\'o}bal},
  journal={Advances in Neural Information Processing Systems},
  volume={35},
  pages={691--702},
  year={2022}
}

@InProceedings{Chen_SEA_OMD,
  title = 	 {Optimistic Online Mirror Descent for Bridging Stochastic and Adversarial Online Convex Optimization},
  author =       {Chen, Sijia and Tu, Wei-Wei and Zhao, Peng and Zhang, Lijun},
  booktitle = 	 {Proceedings of the 40th International Conference on Machine Learning},
  pages = 	 {5002--5035},
  year = 	 {2023},
}

@InProceedings{koloskova2020unified,
  title={A unified theory of decentralized SGD with changing topology and local updates},
  author={Koloskova, Anastasia and Loizou, Nicolas and Boreiri, Sadra and Jaggi, Martin and Stich, Sebastian},
  booktitle={Proceedings of the 37th International conference on machine learning},
  pages={5381--5393},
  year={2020},
}

@InProceedings{khaled_tighter_theory_of_LSGD,
  title = 	 {Tighter Theory for Local SGD on Identical and Heterogeneous Data},
  author =       {Khaled, Ahmed and Mishchenko, Konstantin and Richtarik, Peter},
  booktitle = 	 {Proceedings of the Twenty Third International Conference on Artificial Intelligence and Statistics},
  year = {2020}
}

@InProceedings{patel_OFL_with_bandit_feedback,
  title = 	 {Federated Online and Bandit Convex Optimization},
  author =       {Patel, Kumar Kshitij and Wang, Lingxiao and Saha, Aadirupa and Srebro, Nathan},
  booktitle = 	 {Proceedings of the 40th International Conference on Machine Learning},
  year={2023}
}

@ARTICLE{kwon2023tighter,
  author={Kwon, Dohyeok and Park, Jonghwan and Hong, Songnam},
  journal={IEEE Transactions on Pattern Analysis and Machine Intelligence}, 
  title={Tighter Regret Analysis and Optimization of Online Federated Learning}, 
  year={2023},
  volume={45},
  number={12},
  pages={15772-15789},
  doi={10.1109/TPAMI.2023.3316672}}

@inproceedings{stich2019localsgdconvergesfast,
  title={Local SGD Converges Fast and Communicates Little},
  author={Stich, Sebastian Urban},
  booktitle={International Conference on Learning Representations (ICLR)},
  year={2019}
}

@article{orabona2025modernintroductiononlinelearning,
  title={A modern introduction to online learning},
  author={Orabona, Francesco},
  journal={arXiv preprint arXiv:1912.13213},
  year={2019}
}

@article{hazan2023introductiononlineconvexoptimization,
  title={Introduction to online convex optimization},
  author={Hazan, Elad},
  journal={Foundations and Trends in Optimization},
  volume={2},
  number={3-4},
  pages={157--325},
  year={2016},
  publisher={Emerald Publishing Limited}
}

@InProceedings{FL_McMahana,
  title = 	 {{Communication-Efficient Learning of Deep Networks from Decentralized Data}},
  author = 	 {McMahan, Brendan and Moore, Eider and Ramage, Daniel and Hampson, Seth and Arcas, Blaise Aguera y},
  booktitle = 	 {Proceedings of the 20th International Conference on Artificial Intelligence and Statistics},
  year = 	 {2017}
}

@InProceedings{SCAFFOLD,
  title = 	 {{SCAFFOLD}: Stochastic Controlled Averaging for Federated Learning},
  author =       {Karimireddy, Sai Praneeth and Kale, Satyen and Mohri, Mehryar and Reddi, Sashank and Stich, Sebastian and Suresh, Ananda Theertha},
  booktitle = 	 {Proceedings of the 37th International Conference on Machine Learning},
  year = 	 {2020}
}

@inproceedings{fedprox,
 author = {Li, Tian and Sahu, Anit Kumar and Zaheer, Manzil and Sanjabi, Maziar and Talwalkar, Ameet and Smith, Virginia},
 booktitle = {Proceedings of Machine Learning and Systems},
 title = {Federated Optimization in Heterogeneous Networks},
 year = {2020}
}

@article{OL_survey,
  title={Online learning: A comprehensive survey},
  author={Hoi, Steven CH and Sahoo, Doyen and Lu, Jing and Zhao, Peilin},
  journal={Neurocomputing},
  volume={459},
  pages={249--289},
  year={2021},
  publisher={Elsevier}
}

@article{FL_monograpg_kairouz,
    author = {Kairouz, Peter and McMahan, H. Brendan},
    title = {Advances and Open Problems in Federated Learning},
    journal = {Foundations and Trends in Machine Learning},
    volume = {14},
    number = {1-2},
    pages = {1-210},
    year = {2021},
    month = {06}
}

@article{FL_overview_strategies,
title = {Federated learning: Overview, strategies, applications, tools and future directions},
journal = {Heliyon},
volume = {10},
number = {19},
pages = {e38137},
year = {2024},
issn = {2405-8440},
doi = {https://doi.org/10.1016/j.heliyon.2024.e38137},
author = {Betul Yurdem and Murat Kuzlu and Mehmet Kemal Gullu and Ferhat Ozgur Catak and Maliha Tabassum},
}

@article{Edgecomputingsurvey,
title = {Edge computing: A survey},
journal = {Future Generation Computer Systems},
volume = {97},
pages = {219-235},
year = {2019},
issn = {0167-739X},
doi = {https://doi.org/10.1016/j.future.2019.02.050},
author = {Wazir Zada Khan and Ejaz Ahmed and Saqib Hakak and Ibrar Yaqoob and Arif Ahmed}
}

@Article{conceptdrift,
AUTHOR = {Mehmood, Hassan and Kostakos, Panos and Cortes, Marta and Anagnostopoulos, Theodoros and Pirttikangas, Susanna and Gilman, Ekaterina},
TITLE = {Concept Drift Adaptation Techniques in Distributed Environment for Real-World Data Streams},
JOURNAL = {Smart Cities},
VOLUME = {4},
YEAR = {2021},
NUMBER = {1},
PAGES = {349--371},
ISSN = {2624-6511}
}

\end{document}